\documentclass[english]{article}
\usepackage{geometry}

\usepackage[T1]{fontenc}
\usepackage[latin9]{inputenc}
\usepackage{bm}
\usepackage{amsmath,mathtools}
\usepackage{amssymb}
\usepackage[unicode=true,
 bookmarks=false,
 breaklinks=false,pdfborder={0 0 1},colorlinks=false]
 {hyperref}
\hypersetup{
 colorlinks,citecolor=blue,filecolor=blue,linkcolor=blue,urlcolor=blue}

\makeatletter
\usepackage{amsthm}
\usepackage{cite}  
\usepackage{comment}
\usepackage{natbib}
\usepackage{booktabs}

\usepackage{graphicx}

\usepackage[linesnumbered,ruled,vlined]{algorithm2e}

\SetCommentSty{mycommfont}
\usepackage{algorithmic}

\usepackage{float}
\usepackage{multirow}
 
\usepackage{dsfont}
\usepackage{tcolorbox}

\usepackage{color}
\definecolor{yxc}{RGB}{255,0,0}
\definecolor{yjc}{RGB}{125,0,0}
\definecolor{ytw}{RGB}{255,69,0}
\definecolor{gen}{RGB}{0,0,200}

\allowdisplaybreaks

\newcommand{\mymid}{\,|\,}

\definecolor{yanxi}{RGB}{0,200,100}

\title{Provable Efficiency of Guidance in Diffusion Models \\for General Data Distribution}
\author{Gen Li\footnote{The authors contributed equally.} \thanks{Department of Statistics, The Chinese University of Hong Kong, Hong Kong; Email: \texttt{genli@cuhk.edu.hk}.}\and Yuchen Jiao\footnotemark[1] \thanks{Department of Statistics, The Chinese University of Hong Kong, Hong Kong; Email: \texttt{yuchenjiao@cuhk.edu.hk}.}}
\date{\today}

\makeatother

\begin{document}

\theoremstyle{plain} \newtheorem{lemma}{\textbf{Lemma}}\newtheorem{proposition}{\textbf{Proposition}}\newtheorem{theorem}{\textbf{Theorem}}

\theoremstyle{assumption}\newtheorem{assumption}{\textbf{Assumption}}
\theoremstyle{remark}\newtheorem{remark}{\textbf{Remark}}

\maketitle

\begin{abstract}
Diffusion models have emerged as a powerful framework for generative modeling, with guidance techniques playing a crucial role in enhancing sample quality.
Despite their empirical success, a comprehensive theoretical understanding of the guidance effect remains limited.
Existing studies only focus on case studies, where the distribution conditioned on each class is either isotropic Gaussian or supported on a one-dimensional interval with some extra conditions. 
How to analyze the guidance effect beyond these case studies remains an open question.
Towards closing this gap, we make an attempt to analyze diffusion guidance under general data distributions.
Rather than demonstrating uniform sample quality improvement, which does not hold in some distributions, we prove that guidance can improve the whole sample quality, in the sense that the average reciprocal of the classifier probability decreases with the existence of guidance. 
This aligns with the motivation of introducing guidance.
\end{abstract}

%\tableofcontents

\section{Introduction}

Score-based diffusion models have recently emerged as an expressive and flexible class of generative models, demonstrating competitive performance on image and audio synthesis tasks \citep{sohl2015deep,song2019generative,ho2020denoising,song2020score,song2021maximum,croitoru2023diffusion,ramesh2022hierarchical,rombach2022high,saharia2022photorealistic}. 
These models operate through a forward process, which progressively transforms data from the target distribution into Gaussian noise, and a reverse process that generates samples.
The reverse process typically involves approximating the score function---defined as the gradient of the log-likelihood of noisy distributions---at various scales by training a neural network \citep{hyvarinen2005estimation,ho2020denoising,hyvarinen2007some,vincent2011connection,song2019generative,pang2020efficient}, followed by solving a reverse stochastic differential equation (SDE) associated with the forward process.
Recent studies have rigorously established the convergence of diffusion models, demonstrating that the generated sample distribution approximates the target distribution \citep{lee2022convergence,lee2022convergence2,chen2022sampling,benton2023nearly,chen2023improved,li2024sharp,gupta2024faster,chen2023the,li2024accelerating,li2024adapting,li2024improved,li2024provable,huang2024denoising,cai2025minimax,li2025dimension}.

As diffusion models become a dominant paradigm for generative modeling in domains such as image, video, and audio, the need for principled methods to modulate their output has grown significantly. For instance, when the data comprises multiple classes, one may seek to generate samples specific to a desired class. 
In practice, the standard approach is to use diffusion guidance \citep{dhariwal2021diffusion,ho2021classifier}, a technique that enhances sample quality by incorporating an auxiliary conditional score function. This method combines the model's score estimate with the gradient of the log-probability of samples conditioned on the desired class through a weighted sum, enabling the generation of outputs with high perceptual quality when an appropriate guidance weight is applied.
Reference \citep{karras2024guiding} proposed to use a bad version of the model for guiding diffusion models.

\subsection{Motivation}

Despite the empirical success and widespread adoption of guidance methods, their theoretical foundations remain unexplored. A key question persists: why does guidance improve the quality of samples generated by diffusion models? Existing literature offers partial insights through case studies, analyzing guidance dynamics in limited scenarios such as mixtures of compactly supported distributions or isotropic Gaussian distributions \citep{chidambaram2024does,wu2024theoretical,bradley2024classifier}. 
However, the effect of guidance across general data distributions remains unknown, and we discover that the uniform improvement does not hold even for Gaussian mixture distributions (see Figure~\ref{fig:exp}),
which highlights a significant gap in our understanding.

\subsection{Our Contributions}

Motivated by the above discoveries, this paper investigates the improvement on the average of the reciprocal of classifier probabilities under general data distributions.
We demonstrate that guidance preferentially enhances the generation of samples associated with higher classifier probabilities, which aligns with the primary motivation for adding guidance. 
Specifically, 
%consider the continuous-time counterparts of diffusion models with guidance, 
we prove that the expectation of the reciprocal of classifier probabilities decreases with guidance. 
This metric bears resemblance to the commonly used Inception Score (IS), a standard measure of sample quality \citep{salimans2016improved}, which also considers the expectation of the (logarithmic) function of classifier probabilities.
Furthermore, we extend our analysis to practical implementations, with discrete-errors and score estimation errors. We prove that the discrete-time processes approximate their continuous-time counterparts, ensuring the applicability of our theoretical results in practical settings.

\paragraph{Comparison with prior works when restricted to specific distributions:}
Existing works focus mainly on specific classes of distributions like GMMs, while our work provides a more general theoretical analysis. Here we compare our findings with prior works when restricted to specific distributions.
In \citet{wu2024theoretical}, the authors demonstrate that $p_{c | X_0}(1 | Y_1^w)\ge p_{c | X_0}(1 | Y_1^0)$ holds under specific conditions, while we show that this inequality does not always hold. In addition, \citet{chidambaram2024does} argues that guidance can degrade the performance of diffusion models, as it may introduce mean overshoot and variance shrinkage. In contrast, our result shows that guidance can improve sample quality by generating more samples of high quality. Furthermore, \citet{bradley2024classifier} shows that classifier guidance can not generate samples from $p(x | c)^{\gamma}p(x)^{1-\gamma}$ for GMMs and establishes its connection to an alternative approach, i.e., the single-step predictor-corrector method, whose effectiveness in this specific setting remains unclear. In contrast, we directly analyze and demonstrate the effectiveness of CFG.

\paragraph{Organization.} The organization of this paper is as follows.
Section \ref{sec:background} provides an overview of diffusion models, guidance, and their continuous time limit.
Section \ref{sec:main} presents the main theoretical results and analysis, with detailed proofs included in Section \ref{sec:analysis}.
Finally, we conclude the paper in Section \ref{sec:discussion} with further discussions.

\section{Background}
\label{sec:background}

In this section, we review basics about diffusion models, guidance, and their continuous limit.
Throughout this paper, we shall use $n = 1,\cdots, N$ and $0\le t \le 1$ to denote the discrete and continuous time steps, respectively.

\subsection{Diffusion models}
\label{subsec:diffusion}

Diffusion models are based on a forward process that progressively transforms data from a target distribution into a sequence of increasingly noisy representations.
Starting from $X_0\in\mathbb{R}^{d}$ drawn from the target distribution $p_{\mathsf{data}}$, the forward process evolves as follows:
\begin{subequations}\label{eq:forward}
\begin{align}
	X_0&\sim p_{\mathsf{data}},\\
	X_{n} &= \sqrt{1-\beta_{n}}X_{n-1} + \sqrt{\beta_n} Z_n
	\quad  n = 1,\cdots, N, 
\end{align}
\end{subequations}
where $0<\beta_n<1$ is the step-size, $\{Z_n\}_{1\le n\le N} \overset{\mathrm{i.i.d.}}{\sim} \mathcal{N}(0,I_d)$ is a sequence of independent Gaussian noise vectors.
This process gradually converts the original distribution into standard Gaussian noise as  $n$ increases.

An essential component of score-based diffusion models is the score function, defined as the gradient of the log-probability of the intermediate distributions in the forward process:
\begin{align*}%\label{eq-def-score}
s_n^{\star}(x) := \nabla\log p_{X_n}(x),\quad 1\le n\le N.
\end{align*}
Assuming access to good approximations of the score functions, denoted $s_{n}(x) \approx s_{n}^{\star}(x)$, one can utilize them to reverse the forward process and generate samples resembling the target distribution. The reverse process is governed by:
% it is shown that the following process can generate $Y_n$ with distribution close to $X_n$:
\begin{subequations}\label{eq:reverse}
\begin{align}
	Y_N&\sim \mathcal{N}(0, I_d),\\
	Y_{n-1} &= \frac{1}{\sqrt{1-\beta_{n}}}\big(Y_n+\beta_{n}s_{n}(Y_n)\big) + \sqrt{\beta_{n}} Z_n, 
\end{align}
\end{subequations}
for $n = N,\cdots, 2$,
where $Z_n \overset{\mathrm{i.i.d.}}{\sim} \mathcal{N}(0,I_d)$ denotes another sequence of independent Gaussian noise vectors.
This reverse process has been shown to gradually remove noise and guide the system back toward the target distribution, in the sense that the generated $Y_n$ has distribution close to that of $X_n$ in \eqref{eq:forward}.

\subsection{Guidance}

%The above framework can be generalized to conditional sampling easily by considering:

%In addition to unconditional diffusion model approximating $p_{\mathsf{data}}$,
Conditional diffusion models are designed to sample from the conditional distributions 
$p(\cdot|c)$, where $c$ represents a specific class label.
This can be achieved by generalizing the unconditional diffusion model defined in \eqref{eq:reverse}, replacing $s_n(Y_n)$ with $s_n(Y_n|c)$, as shown below:
%It seeks to sample from the conditional distribution $p(\cdot|c)$, which can be generalized from \eqref{eq:reverse} by replacing $s_n(Y_n)$ with $s_n(Y_n|c)$ as follows:
\begin{subequations}\label{eq:condition}
\begin{align}
	Y_N&\sim \mathcal{N}(0, I_d),\\
	Y_{n-1} &= \frac{1}{\sqrt{1-\beta_{n}}}\big(Y_n+\beta_{n}s_{n}(Y_n\mymid c)\big) + \sqrt{\beta_{n}} Z_n,
\end{align}
\end{subequations}
for $n = N,\cdots, 2$,
where $s_{n}(x|c)$ are good estimates of the gradient of the log-density function $p_{X_n\mymid c}$, given the condition $c$. That is, $s_{n}(x|c)\approx s_{n}^{\star}(x|c) = \nabla\log p_{X_n\mymid c}(x\mymid c)$.
The noise terms $Z_n \overset{\mathrm{i.i.d.}}{\sim} \mathcal{N}(0,I_d)$ represent a sequence of independent Gaussian noise vectors.

%Later on, researchers propose to add some guidance with the hope to increase $p_{c\mymid X_0}(c\mymid Y_0)$ as following:
To further enhance the quality of conditional sampling, researchers introduced guidance techniques. These methods aim to increase the posterior probability $p_{c\mymid X_0}(c\mymid Y_0)$ by modifying the reverse process as follows:
\begin{align}\label{eq:guidance-1}
	Y_{n-1}^w &= \frac{1}{\sqrt{1-\beta_{n}}}\big(Y_n^w+\beta_{n}(s_{n}(Y_n^w\mymid c)  + w\nabla\log p_{c\mymid X_n}(c\mymid Y_n^w))\big) + \sqrt{\beta_{n}} Z_n,
\end{align}
where the guidance scale $w$ controls the strength of the modification.
Furthermore, reverse process \eqref{eq:guidance-1} can be approximated as
\begin{align}
\label{eq:guidance}
Y_{n-1}^w &= \frac{1}{\sqrt{1-\beta_{n}}}\big(Y_n^w+\beta_{n}((1+w)s_{n}(Y_n^w\mymid c) - ws_{n}(Y_n^w)\big) + \sqrt{\beta_{n}} Z_n.
\end{align}
This approximation is derived from the observation that $\nabla\log p_{c\mymid X_n}(c\mymid x) = s_{n}^{\star}(x\mymid c) - s_{n}^{\star}(x)$, which is referred to as classifier free guidance \citep{ho2021classifier}.
%where the second line is from the observation that $\nabla\log p_{c\mymid X_n}(c\mymid x)) = s_{n}^{\star}(x\mymid c) - s_{n}^{\star}(x)$, and is called classifier free guidance \citep{ho2021classifier} and $w$ denotes the guidance scale.

%Here, $p_{X_n\mymid c}$ denotes the density function of $X_n$ when given the condition $c$.

\subsection{Continuous time limit}

%It is found that the diffusion process in \eqref{eq:forward} has a nice correspondence with its following continuum limit:
The discrete-time diffusion process described in Section \ref{subsec:diffusion} exhibits a natural correspondence to its continuous-time counterpart.
Specifically, the forward process corresponds to the following stochastic differential equation (SDE):
\begin{subequations}
\label{eq:SDE}
\begin{align}
%\mathrm{d}X_t = -\frac{1}{2}\beta(t)X_t\mathrm{d}t + \sqrt{\beta(t)}\mathrm{d}B_t,
\mathrm{d}X_t &= -\frac{1}{2(1-t)}X_t\mathrm{d}t + \frac{1}{\sqrt{1-t}}\mathrm{d}B_t,\\
&\text{with }X_0 \sim p_{\mathsf{data}},
\quad\text{for }0 \le t \le 1-\delta,\notag
\end{align}
where $B_t$ denotes the standard Brownian motion, and $\delta>0$ can be arbitrarily small. 
It transforms the data distribution into a standard Gaussian distribution as $t\to 1$.
Similarly, the reverse process in \eqref{eq:condition} corresponds to the following continuous-time SDE:
\begin{align}
%\mathrm{d}Y_t = \Big(\frac{1}{2}Y_t + \nabla\log p_{X_{T-t}}(Y_t)\Big)\beta(T-t)\mathrm{d}t + \sqrt{\beta(T-t)}\mathrm{d}B_t,
\mathrm{d}Y_t &= \Big(\frac{1}{2}Y_t + \nabla\log p_{X_{1-t}\mymid c}(Y_t\mymid c)\Big)\frac{\mathrm{d}t}{t} + \frac{1}{\sqrt{t}}\mathrm{d}B_t,
%\quad\text{with }Y_{\delta} \sim p_{X_{1-\delta}\mymid c},
\quad \text{for }\delta \le t \le 1.\notag
\end{align}
\end{subequations}
This reverse SDE effectively transforms the noise distribution back toward the target distribution conditioned on $c$, guided by the conditional score function $\nabla\log p_{X_{1-t}\mymid c}(Y_t\mymid c)$.
%Here, we use $B_t$ to denote the standard Brownian motion, and $\delta > 0$ can be arbitrarily small.
If the initialization $Y_{\delta} \sim p_{X_{1-\delta}\mymid c}$, it is well-known that $Y_t$ has the same distribution with the reverse process of $X_t$,
which is stated in the following lemma:
\begin{lemma} \label{lem:cont}
It can be shown that for $0 \le \tau \le t \le 1-\delta$,
\begin{align} \label{eq:cont-X}
X_t\mymid X_{\tau} \sim \mathcal{N}\bigg(\sqrt{\frac{1-t}{1-\tau}}X_{\tau}, \frac{t-\tau}{1-\tau}I\bigg),
\end{align}
and if $Y_{\delta} \sim p_{X_{1-\delta}\mymid c}$, then
\begin{align} \label{eq:cont-XY}
\{Y_t\} \overset{\mathrm{d}}{=} \{X_{1-t}\},
\quad\text{for }\delta \le t \le 1.
\end{align}
\end{lemma}

The above result can be found in~\citet{song2020score}.
When extending this framework to conditional sampling with guidance in \eqref{eq:guidance}, the reverse SDE becomes
%In addition, when considering condition sampling with guidance, it becomes
\begin{align} \label{eq:cont-guidance}
\mathrm{d}Y_t^w &= \Big(\frac{1}{2}Y_t^w + (1+w)\nabla\log p_{X_{1-t}\mymid c}(Y_t^w\mymid c)  - w\nabla\log p_{X_{1-t}}(Y_t^w)\Big)\frac{\mathrm{d}t}{t} + \frac{1}{\sqrt{t}}\mathrm{d}B_t.
\end{align}
The continuous-time framework provides a powerful perspective for understanding and analyzing score-based diffusion models.

\section{Main results}
\label{sec:main}

In this section, we shall present our main theorem and its proof.
For the reverse process with guidance~\eqref{eq:cont-guidance}, 
we prove that after introducing a non-zero guidance into the diffusion process, the expectation of a specific decreasing function of the classifier probability will decrease as $t$ increases.
This is formally stated in the following theorem.

\begin{theorem} \label{thm:main}
Let 
\begin{align}
\phi_t(y) := p_{c\mymid X_{1-t}}(c\mymid y)^{-1}
\end{align}
which is a decreasing map of $p_{c\mymid X_{1-t}}(c\mymid y)$.
It can be shown that for any $\delta < t < 1$,
\begin{align}
&\quad \phi_{t}(Y_{t}^w) - \mathbb{E}\big[\phi_{t+\mathrm{d}t}(Y_{t+\mathrm{d}t}^w) \mymid Y_t^w\big]= \frac{w}{t}p_{c\mymid X_{1-t}}(c\mymid Y_t^w)^{-1}\Big\|\nabla\log p_{X_{1-t}\mymid c}(Y_t^w\mymid c) - \nabla\log p_{X_{1-t}}(Y_t^w)\Big\|_2^2\mathrm{d}t,
\end{align}
where $Y_t^w$ is defined in~\eqref{eq:cont-guidance}.
\end{theorem}

The above result reveals that the average reciprocal of classifier probability $p_{c\mymid X_{1-t}}(c\mymid y)^{-1}$ decreases when we add non-zero guidance.
When compared with the case without guidance, that is $w=0$, the total expected improvement over the diffusion process is given by:
\begin{align}
&\int \frac{w}{t}p_{c\mymid X_{1-t}}(c\mymid Y_t^w)^{-1}\Big\|\nabla\log p_{X_{1-t}\mymid c}(Y_t^w\mymid c) - \nabla\log p_{X_{1-t}}(Y_t^w)\Big\|_2^2\mathrm{d}t.
\end{align}
This result reflects an improvement in sample quality, as samples with higher classifier probabilities are favored.

The choice of $p_{c\mymid X_{1-t}}(c\mymid y)^{-1}$ in our analysis is primarily for technical considerations.
It rewards more on the decrease of bad samples with small $p_{c\mymid X_{1-t}}(c\mymid y)$, which means it places greater emphasis on reducing the probability of generating low-quality or misclassified samples. 
This aligns with the initial motivation of introducing guidance.
In practice, Inception Score (IS) is commonly employed to measure sample quality, which is related to the average logarithm of the classifier probability $\mathbb{E}[\log p_{c\mymid X_{1-t}}(c\mymid y)]$.
This is conceptually aligned with the metric in our analysis, with the difference being that IS adopts $\log p_{c\mymid X_{1-t}}(c\mymid y)$ as the weight, while we use $p_{c\mymid X_{1-t}}(c\mymid y)^{-1}$,
but both aim to increase the ratio of high-quality samples (measured by the classifier probability).
In addition, to address potential concerns, we note that although some practical limitations of IS have been identified \citep{barratt2018note}, it remains a commonly used metric for evaluating sample quality in the study of diffusion guidance \citep{dhariwal2021diffusion,ho2021classifier}. Moreover, in our theoretical analysis, we use the true conditional probability, which addresses the estimation issues discussed in \citet{barratt2018note}.

Theorem \ref{thm:main} states that guidance improves the averaged reciprocal of the classifier probability rather than the classifier probability of each individual sample. This suggests that while guidance improves overall sample quality, it may lead to a decline in quality for a small subset of samples. This insight encourages the development of adaptive guidance methods that address this issue and achieve more uniform performance gains, which is a potential practical application of our theory.

Our main result is established through the following key observation, whose proof can be found in Section \ref{subsec:proof-lem-invariance}. 

\begin{lemma} \label{lem:invariance}
For any $\varepsilon > 0$ and $0 \le \tau \le t \le 1-\varepsilon$, we have
\begin{subequations} \label{eq:invariance}
\begin{align}
p_{c\mymid X_t}(c\mymid x)^{-1} = \mathbb{E}_{x_{\tau} \sim X_{\tau}}\big[p_{c\mymid X_{\tau}}(c\mymid x_{\tau})^{-1} \mymid X_t = x\big],
\end{align}
or equivalently, for any $\varepsilon \le \tau \le t \le 1$,
\begin{align}
p_{c\mymid X_{1-\tau}}(c\mymid y)^{-1} = \mathbb{E}_{y_t \sim Y_t}\big[p_{c\mymid X_{1-t}}(c\mymid y_t)^{-1} \mymid Y_{\tau} = y\big],
\end{align}
where, $X_t$ and $Y_t$ are defined in~\eqref{eq:SDE}.
\end{subequations}
\end{lemma}

%Analysis of Theorem \ref{thm:main} relies on the martingale property stated in Lemma \ref{lem:invariance}, which relies heavily on the property of $p_{X_\tau|X_t}$ in DDPMs and can not be applied for DDIMs. Thus extending our framework to DDIMs remains an open question due to the absence of this key property.

With Lemma \ref{lem:invariance} in hand, we are ready to prove our main theorem.
Before diving into the proof details, we would like to first explain the main analysis idea: First, this result comes from the key observation that the function of reverse process, $p_{c|X_{t}}(c|X_t)^{\rm -1}$, forms a martingale, as stated in Lemma \ref{lem:invariance}, which is established through a careful decomposition of $p_{c|X_{t}}$ and $p_{X_{\tau}|X_{t}}$. Next, the guidance term $s_t(x|c) - s_t(x)$ in classifier-free guidance (CFG) aligns with the direction of $-\nabla p_{c| X_t}(c| x)^{-1} = p_{c| X_t}(c| x)^{-1}[s_t(x|c) - s_t(x)]$, which makes us expect that adding the guidance at time $t$ can decrease $\mathbb{E}_{x_{\tau} \sim X_{\tau}}\big[p_{c| X_{\tau}}(c| x_{\tau})^{-1} | X_t = x\big]$ for all $\tau \le t$. Finally, to achieve the desired result, particular care must be taken in handling first- and second-order differential terms with respect to $t$ for the process $p_{c| X_{1-t}}(c| Y_t^w)^{-1}$ due to its randomness nature, which is completed in the following based on the technique of Ito's formula.

\begin{proof}[Proof of Theorem~\ref{thm:main}]
The relation~\eqref{eq:invariance} in the above lemma gives us
\begin{align}
0 &= \frac{1}{\delta}\Big\{\mathbb{E}\big[p_{c\mymid X_{1-t-\delta}}(c\mymid Y_{t+\delta})^{-1} - p_{c\mymid X_{1-t}}(c\mymid Y_t)^{-1} \mymid Y_{t} = y_{t}\big]\Big\} \notag\\
&= \frac{\partial p_{c\mymid X_{1-t}}(c\mymid y)^{-1}}{\partial t}\mymid_{y = y_t} \!+\! \frac{1}{2t}\mathsf{Tr}\Big(\nabla^2 p_{c\mymid X_{1-t}}(c\mymid y_{t})^{-1}\Big) \notag\\
&\quad+ \nabla p_{c\mymid X_{1-t}}(c\mymid y_{t})^{-1}
 \Big(\Big(\frac{1}{2}y_{t} + \nabla\log p_{X_{1-t}\mymid c}(y_{t}\mymid c)\Big)\frac{1}{t}\Big)+ O(\delta), \label{eq:diff}
\end{align}
where the second relation is established in Section~\ref{sec:proof-diff}.
Here, we let $\delta > 0$ be some small quantity, which depends only on $y_t, t$ and the property of $X_0$.
Similarly, we have
\begin{align}
&\frac{1}{\delta}\Big\{\mathbb{E}\big[p_{c\mymid X_{1-t-\delta}}(c\mymid Y_{t+\delta}^w)^{-1} - p_{c\mymid X_{1-t}}(c\mymid Y_t^w)^{-1} \mymid Y_{t}^w = y_{t}\big]\Big\} \notag\\
&= \frac{\partial p_{c\mymid X_{1-t}}(c\mymid y)^{-1}}{\partial t}\mymid_{y = y_t} + \frac{1}{2t}\mathsf{Tr}\Big(\nabla^2 p_{c\mymid X_{1-t}}(c\mymid y_{t})^{-1}\Big) \notag\\
&+ \nabla p_{c\mymid X_{1-t}}(c\mymid y_{t})^{-1}\Big(\Big(\frac{1}{2}y_{t} + (1+w)\nabla\log p_{X_{1-t}\mymid c}(y_{t}\mymid c)
- w\nabla\log p_{X_{1-t}}(y_{t})\Big)\frac{1}{t}\Big)
+ O(\delta).
\end{align}

Comparing the above two relations leads to
\begin{align}
&\quad\mathbb{E}\big[\phi_{t+\delta}(Y_{t+\delta}^w) \mymid Y_t^w\big] - \phi_{t}(Y_{t}^w) \notag\\
&= \delta\frac{w}{t} \Big(\nabla\log p_{X_{1-t}\mymid c}(Y_{t}^w\mymid c) - \nabla\log p_{X_{1-\tau}}(Y_{t}^w)\Big) \nabla p_{c\mymid X_{1-t}}(c\mymid Y_{t}^w)^{-1} + O(\delta^2)\notag \\
&= -\delta\frac{w}{t} p_{c\mymid X_{1-t}}(c\mymid Y_{t}^w)^{-1}\Big\|\nabla\log p_{X_{1-t}\mymid c}(Y_{t}^w\mymid c)- \nabla\log p_{X_{1-t}}(Y_{t}^w)\Big\|_2^2
+ O(\delta^2),
\end{align}
where the second relation holds since
\begin{align}
&\nabla p_{c\mymid X_{1-t}}(c\mymid y)^{-1} = -p_{c\mymid X_{1-t}}(c\mymid y)^{-1}
\Big(\nabla\log p_{X_{1-t}\mymid c}(y\mymid c) - \nabla\log p_{X_{1-t}}(y)\Big).
\end{align}
Then we can conclude the proof here.

\end{proof}

\begin{figure*}[t]
\vskip 0.2in
\begin{center}
\centerline{\includegraphics[width=0.9\textwidth]{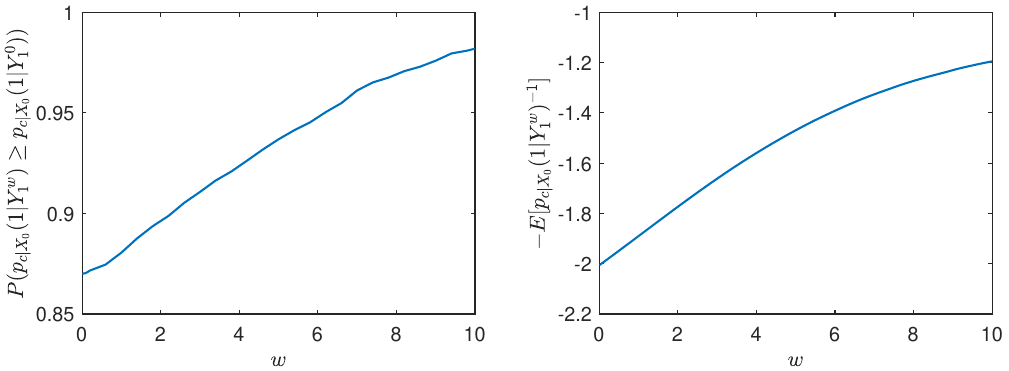}}
\caption{Experimental results on GMM. left: Ratio of samples with improved classifier probabilities for different guidance scales $w$; right: Expectation of $-p_{c\mymid X_0}(1\mymid Y_1^w)^{-1}$ for varying $w$.}
\label{fig:exp}
\end{center}
\vskip -0.2in
\end{figure*}

\begin{figure*}[t]
\vskip 0.2in
\begin{center}
\centerline{\includegraphics[width=0.9\textwidth]{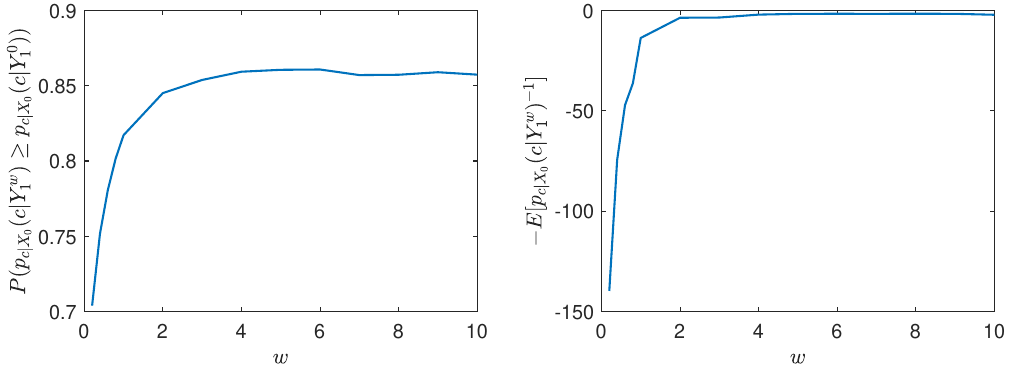}}
\caption{Experimental results on ImageNet dataset. left: Ratio of samples with improved classifier probabilities for different guidance scales $w$; right: Expectation of $-p_{c\mymid X_0}(1\mymid Y_1^w)^{-1}$ for varying $w$.}
\label{fig:exp-imagenet}
\end{center}
\vskip -0.2in
\end{figure*}

\subsection{Numerical validation}
\label{sec:toy-example}

In this section, we present experimental results on the Gaussian Mixture Model (GMM) and ImageNet dataset to demonstrate that guidance does not uniformly enhance the quality of all samples. Instead, it improves overall sample quality by reducing the average reciprocal of the classifier probability. This observation empirically validate our theoretical findings.

%\subsubsection{Gaussian Mixture Model}

\paragraph{Gaussian Mixture Model:}
Let us consider a distribution with two classes $c = 0, 1$, each with equal prior probability $p_c(0) = p_c(1)=0.5$, in a one-dimensional data space ($d=1$).
The data distribution is defined as follows:
\begin{subequations}
\begin{align*}
X_0\mymid c=0 &~~\sim~~ \mathcal{N}(0, 1)\\
X_0\mymid c=1 &~~\sim~~ \frac{1}{2}\mathcal{N}(1, 1) + \frac{1}{2}\mathcal{N}(-1, 1).
\end{align*}
\end{subequations}
According to the DDPM framework with guidance \eqref{eq:guidance}, the reverse process adopts the following update rule. Starting from $Y_N^w\sim\mathcal{N}(0,1)$, the process evolves for $n=N,\cdots,2$: 
%as:
\begin{align}
\label{eq:DDPM}
	Y_{n-1}^{w} &\!=\! \frac{1}{\sqrt{\alpha_{n}}}\big(Y_n^{w}+(1-\alpha_n)\big[- w\nabla \log p_{X_{1-\overline{\alpha}_n}}(Y_n^{w}) \notag\\
    &~+(1+w)\nabla \log p_{X_{1\!-\!\overline{\alpha}_n}\mymid c}(Y_n^{w}\mymid c)\big]\big) + \sqrt{1-\alpha_n} Z_n,
\end{align}
where $Z_n \overset{\mathrm{i.i.d.}}{\sim} \mathcal{N}(0,1)$ is a sequence of independent Gaussian random variables.

Here, we focus on the conditional class $c = 1$. The score functions $\nabla \log p_{X_{1-\overline{\alpha}_n}\mymid c}(x\mymid 1)$, $\nabla \log p_{X_{1-\overline{\alpha}_n}}(x)$, and the classifier probability $p_{c\mymid X_{1-\overline{\alpha}_n}}(1\mymid x)$ are provided in Appendix \ref{appendix:GMM} (cf. \eqref{eq:score-GMM-1}, \eqref{eq:score-GMM-2}, and \eqref{eq:llh-GMM}).
To empirically validate our theoretical findings, we simulate the DDPM framework under different guidance scales $w$.
%In addition, we also show the decrease the average reciprocal of the classifier, which verifies Theorem \ref{thm:main}.
Specifically, we fix $N=4000$, vary $w$ from $0.01$ to $10$, and perform $10^4$ trials for each $w$.
We compute $Y_1^w$ by implementing the reverse process in \eqref{eq:DDPM}, and its counterpart $Y_1^0$ without guidance.
For each trial, we evaluate classifier probability 
$p_{c\mymid X_{0}}(1\mymid Y_1^w)$ and $p_{c\mymid X_{0}}(1\mymid Y_1^0)$, and compute the empirical probability of $P(p_{c\mymid X_{0}}(1\mymid Y_1^w)\ge p_{c\mymid X_0}(1\mymid Y_1^0))$.
In addition, we also calculate the average of $-p_{c\mymid X_0}(1\mymid Y_1^w)^{-1}$ for various $w$.
The results are shown in Figure \ref{fig:exp}.

\paragraph{ImageNet dataset:}
We conduct a numerical experiment on the ImageNet dataset.
Specifically, we generate samples using a pre-trained diffusion model \citep{diffusioncode} with varying values of the guidance level $w$, and evaluate the classifier probabilities using the Inception v3 classifier \citep{szegedy2016rethinking}.
We compute two statistics: $P(p_{c\mymid X_{0}}(1\mymid Y_1^w)\ge p_{c\mymid X_0}(1\mymid Y_1^0))$ and $-\mathbb{E}[p_{c\mymid X_0}(1\mymid Y_1^w)^{-1}]$, averaged over $20000$ random trials --- $20$ trials for each of the $1000$ ImageNet categories.
The experimental results are presented in Figure \ref{fig:exp-imagenet}. 

It is observed that the empirical probability $P(p_{c\mymid X_0}(1\mymid Y_1^w)\ge p_{c\mymid X_0}(1\mymid Y_1^0))$ is less than $1$ for any $w<10$, which indicates the guidance does not achieve uniform improvement in classifier probabilities.
However, the average of $-p_{c\mymid X_0}(1\mymid Y_1^w)^{-1}$ increases with $w$, which explains why guidance effectively enhances sample quality, as predicted by Theorem~\ref{thm:main}.
Moreover, we remark that the performance of diffusion models is commonly evaluated by two metrics in practice: diversity and sample quality. This study primarily focuses on the sample quality measured in a similar way as the Inception Score, which increases with $w$. However, prior work \citet{ho2021classifier} has demonstrated that large values of $w$ can significantly reduce sample diversity, leading to unsatisfactory performance in real-world applications.

%\subsubsection{ImageNet dataset}

\section{Analysis}
\label{sec:analysis}

In this section, we shall provide details in the proof of main results.

\subsection{Proof of Lemma~\ref{lem:invariance}}
\label{subsec:proof-lem-invariance}

According to the equivalence between $X_t$ and $Y_t$ (see~\eqref{eq:cont-XY} in Lemma~\ref{lem:cont}), it is sufficient to focus on the first relation.
Recalling Lemma~\ref{lem:cont} again tells us
\begin{align*}
&\quad\mathbb{E}_{x_{\tau} \sim X_{\tau}}\big[p_{c\mymid X_{\tau}}(c\mymid x_{\tau})^{-1} \mymid X_t = x\big]\notag\\
&= \int_{x_{\tau}} p_{X_{\tau}\mymid X_t, c}(x_{\tau}\mymid x, c) p_{c\mymid X_{\tau}}(c\mymid x_{\tau})^{-1}\mathrm{d}x_{\tau} \\
&= \int_{x_{\tau}} \frac{p_{X_{\tau}\mymid c}(x_{\tau}\mymid c)(2\pi\sigma^2)^{-d/2}\exp(-\frac{\|x-\alpha x_{\tau}\|_2^2}{2\sigma^2})}{p_{X_t\mymid c}(x\mymid c)} \notag\\
&\quad\cdot \frac{p_{X_{\tau}}(x_{\tau})}{p_{X_{\tau}\mymid c}(x_{\tau}\mymid c)p_{c}(c)}\mathrm{d}x_{\tau} \\
&= \frac{\int_{x_{\tau}} p_{X_{\tau}}(x_{\tau})(2\pi\sigma^2)^{-d/2}\exp(-\frac{\|x-\alpha x_{\tau}\|_2^2}{2\sigma^2})\mathrm{d}x_{\tau}}{p_{X_t\mymid c}(x\mymid c)p_{c}(c)} \\
&= \frac{p_{X_t}(x)}{p_{X_t\mymid c}(x\mymid c)p_{c}(c)} = p_{c\mymid X_t}(c\mymid x)^{-1},
\end{align*}
where we let $\alpha = \sqrt{\frac{1-t}{1-\tau}}$ and $\sigma = \sqrt{\frac{t-\tau}{1-\tau}}$.
Here, the first line is just the definition of conditional expectation; 
the second line comes from the Bayes rule and the relation~\eqref{eq:cont-X};
and the last line can be derived by applying the Bayes rule and the relation~\eqref{eq:cont-X} again.

\subsection{Preliminary analysis of $p_{c\mymid X_{1-t}}$}

We begin by establishing some key properties of $p_{c\mymid X_{1-t}}$ to support the proofs of our main results.
%\paragraph{Basic calculations of derivative of $p_{c\mymid X_{1-t}}$.}
Let $R < \infty$ be some quantity such that
\begin{align}\label{eq:def-R}
\mathbb{P}(\|X_0\|_2 \!<\! R) > \frac{1}{2}
\quad\text{and}\quad
\mathbb{P}(\|X_0\|_2\! <\! R\mymid c) > \frac{1}{2}.
\end{align}
Then there exists some quantity $C_{t, k, R} > 0$ depending only on $t, k, R$, such that the following bounds hold:
\begin{subequations}
\begin{align}
&\nabla^k p_{c\mymid X_{1-t}}(c\mymid y)^{-1} \!\le\! \exp(C_{t, k, R}(1 \!+\! \|y\|_2^2)); \label{eq:bound-1}\\
&\frac{\partial^k p_{c\mymid X_{1-t}}(c\mymid y)^{-1}}{\partial t^k} \!\le\! \exp(C_{t, k, R}(1 \!+\! \|y\|_2^2)); \label{eq:bound-2}\\
&\nabla^k \frac{\partial p_{c\mymid X_{1-t}}(c\mymid y)^{-1}}{\partial t} \!\le\! \exp(C_{t, k, R}(1 \!+\! \|y\|_2^2)),\label{eq:bound-3}
\end{align}
\end{subequations}
where $\nabla^k p_{c\mymid X_{1-t}}(c\mymid y)^{-1}$ denotes the $k$-th order gradient with respect to $y$ of function $p_{c\mymid X_{1-t}}(c\mymid y)^{-1}$.

In the following, we focus primarily on the gradient $\nabla p_{c\mymid X_{1-t}}(c\mymid y)^{-1}$,
as the other bounds can be derived using similar techniques.
Notice that $\nabla p_{c\mymid X_{1-t}}(c\mymid y)^{-1}$ satisfies the following decomposition:
\begin{align}
&\quad\nabla p_{c\mymid X_{1-t}}(c\mymid y)^{-1} \nonumber\\
&= -p_{c\mymid X_{1-t}}(c\mymid y)^{-2}\nabla p_{c\mymid X_{1-t}}(c\mymid y) \notag\\
&= -p_{c\mymid X_{1-t}}(c\mymid y)^{-1}\nabla \log p_{c\mymid X_{1-t}}(c\mymid y) \notag\\
&= p_{c\mymid X_{1-t}}(c\mymid y)^{-1}\nabla\big[\log p_{X_{1-t}}(y) - \log p_{X_{1-t}\mymid c}(y\mymid c)\big].\label{eq:bound-grad}
\end{align}
In addition, it can be shown later that
\begin{subequations}
\begin{align} \label{eq:classifier-ub}
p_{c\mymid X_{1-t}}(c\mymid y)^{-1} 
\le 2p_c(c)^{-1}\exp\Big(\frac{(\|y\|_2+\sqrt{t}R)^2}{2(1-t)}\Big),
\end{align}
and
\begin{align}
\big\|\nabla\log p_{X_{1-t}}(y)\big\|_2 &\lesssim \frac{\|y\|_2 + \sqrt{t}R}{1-t} + \frac{d}{\sqrt{1-t}}, \label{eq:score_bound}\\
\big\|\nabla\log p_{X_{1-t}\mymid c}(y)\big\|_2 &\lesssim \frac{\|y\|_2 + \sqrt{t}R}{1-t} + \frac{d}{\sqrt{1-t}},\label{eq:score_bound-2}
\end{align}
\end{subequations}
where $f\lesssim g$ implies that there exists a universal constant $C > 0$ such that $f \le Cg$. 
By inserting \eqref{eq:classifier-ub} and \eqref{eq:score_bound} into \eqref{eq:bound-grad},
the gradient $\nabla p_{c\mymid X_{1-t}}(c\mymid y)^{-1}$ can be controlled directly.
%if we can establish the above two claims.

\paragraph{Proof of Claim~\eqref{eq:classifier-ub} - \eqref{eq:score_bound-2}.}
We begin with establishing \eqref{eq:classifier-ub}.
First, according to Lemma \ref{lem:cont}, random variable $X_{1-t}|X_0$ follows Gaussian distribution $\mathcal{N}(\sqrt{t}X_0,(1-t)I)$. Thus we have 
\begin{align}
&p_{X_{1-t}}(y) = \int_{x_0} p_{X_0}(x_0)p_{X_{1-t}|X_0}(y|x_0)\mathrm{d}x_0\notag\\
&= \int_{x_0} p_{X_0}(x_0)(2\pi(1-t))^{-d/2}\exp\left(-\frac{\|y-\sqrt{t}x_0\|_2^2}{2(1-t)}\right)\mathrm{d}x_0 \notag\\
&\le (2\pi(1-t))^{-d/2}\int_{x_0} p_{X_0}(x_0)\mathrm{d}x_0\nonumber\\
&= (2\pi(1-t))^{-d/2}.\label{eq:proof-classifier-ub-1}
\end{align}
Moreover, recalling the definition of $R$ in \eqref{eq:def-R}, we have
\begin{align}
&p_{X_{1-t}\mymid c}(y\mymid c) 
\ge p_{X_{1-t},\|X_0\|_2 < R\mymid c}(y\mymid c) \notag\\
& = \mathbb{P}(\|X_0\|_2 < R\mymid c)p_{X_{1-t}\mymid c,\|X_0\|_2 < R}(y\mymid c,\|X_0\|_2 < R)\notag\\
&\ge \frac{1}{2}\inf_{x_0: \|x_0\|_2 < R} (2\pi(1-t))^{-d/2}\exp\left(-\frac{\|y-\sqrt{t}x_0\|_2^2}{2(1-t)}\right) \\
&\ge \frac{1}{2} (2\pi(1-t))^{-d/2}\exp\left(-\frac{(\|y\|_2+\sqrt{t}R)^2}{2(1-t)}\right),\label{eq:proof-classifier-ub-2}
\end{align}
where  $p_{X_{1-t},\|X_0\|_2 < R\mymid c}(y\mymid c)$ denotes the joint probability density of $X_{1-t}$ and the binary random variable indicating $\|X_0\|_2 < R$ or not, and $p_{X_{1-t}\mymid c,\|X_0\|_2 < R}(y\mymid c)$ denotes the probability density of $X_{1-t}$ conditioned on the class label $c$ and $\|X_0\|_2 < R$.
Combining \eqref{eq:proof-classifier-ub-1} and \eqref{eq:proof-classifier-ub-2}, we have
\begin{align*}
p_{c\mymid X_{1-t}}(c\mymid y)^{-1} 
&= \frac{p_{X_{1-t}}(y)}{p_c(c)p_{X_{1-t}\mymid c}(y\mymid c)}\\
&\le 2p_c(c)^{-1}\exp\Big(\frac{(\|y\|_2+\sqrt{t}R)^2}{2(1-t)}\Big).
\end{align*}

Next, we shall prove \eqref{eq:score_bound}. For $t < 1$, recalling that the
random variable $X_{1-t}|X_0$ follows Gaussian distribution $\mathcal{N}(\sqrt{t}X_0,(1-t)I)$, 
the score function has the following expression
\begin{align}
&\quad \nabla\log p_{X_{1-t}}(y)\notag\\
&= -p_{X_{1-t}}(y)^{-1}\int_{x_0} p_{X_0}(x_0)(2\pi(1-t))^{-d/2}\notag\\ 
&\qquad\qquad \cdot \exp\Big(-\frac{\|y - \sqrt{t}x_0\|_2^2}{2(1-t)}\Big)\frac{y - \sqrt{t}x_0}{1-t}\mathrm{d}x_0 \notag\\
&= -\int_{x_0} p_{X_0\mymid X_{1-t}}(x_0\mymid y)\frac{y - \sqrt{t}x_0}{1-t}\mathrm{d}x_0.
\end{align}
Moreover, noticing that for any $D>0$,
\begin{align*}
&\quad\big\|\nabla\log p_{X_{1-t}}(y)\big\|_2\notag\\
&=\int_{x_0:\big\|\frac{y - \sqrt{t}x_0}{\sqrt{1-t}}\big\|_2 \le D} p_{X_0\mymid X_{1-t}}(x_0\mymid y)\left\|\frac{y - \sqrt{t}x_0}{1-t}\right\|_2\mathrm{d}x_0\notag\\
&\quad + \Big\|p_{X_{1-t}}(y)^{-1}\int_{x_0:\big\|\frac{y - \sqrt{t}x_0}{\sqrt{1-t}}\big\|_2 > D} p_{X_0}(x_0)(2\pi(1-t))^{-d/2}\notag\\
&\qquad\cdot\exp\Big(-\frac{\|y - \sqrt{t}x_0\|_2^2}{2(1-t)}\Big)\frac{y - \sqrt{t}x_0}{1-t}\mathrm{d}x_0\Big\|_2.
%\int_{x_0:\big\|\frac{y - \sqrt{t}x_0}{\sqrt{1-t}}\big\|_2 > D} p_{X_0\mymid X_{1-t}}(x_0\mymid y)\frac{y - \sqrt{t}x_0}{1-t}\mathrm{d}x_0.
\end{align*}
For the first term, we have
\begin{align*}
&\int_{x_0:\big\|\frac{y - \sqrt{t}x_0}{\sqrt{1-t}}\big\|_2 \le D} p_{X_0\mymid X_{1-t}}(x_0\mymid y)\left\|\frac{y - \sqrt{t}x_0}{1-t}\right\|_2\mathrm{d}x_0\le \frac{D}{\sqrt{1-t}}.
\end{align*}
For the second term, noticing that
\begin{align*}
p_{X_{1-t}}(y) 
%&\ge \frac{1}{2}\inf_{x_0:\|x_0\|_2 < R}(2\pi(1-t))^{-d/2}\exp\Big(-\frac{\|y - \sqrt{t}x_0\|_2^2}{2(1-t)}\Big) \\
&\ge \frac12(2\pi(1-t))^{-d/2}\exp\left(-\frac{(\|y\|_2 + \sqrt{t}R)^2}{2(1-t)}\right),
\end{align*}
we have
\begin{align*}
&\quad\Big\|p_{X_{1-t}}(y)^{-1}\int_{x_0:\big\|\frac{y - \sqrt{t}x_0}{\sqrt{1-t}}\big\|_2 > D} p_{X_0}(x_0)(2\pi(1-t))^{-d/2}\notag\\
&\qquad\cdot \exp\Big(-\frac{\|y - \sqrt{t}x_0\|_2^2}{2(1-t)}\Big)\frac{y - \sqrt{t}x_0}{1-t}\mathrm{d}x_0\Big\|_2\notag\\
&\le  2\exp\Big(\frac{(\|y\|_2 + \sqrt{t}R)^2}{2(1-t)}\Big)\int_{x_0:\big\|\frac{y - \sqrt{t}x_0}{\sqrt{1-t}}\big\|_2 > D} p_{X_0}(x_0)\notag\\ 
&\qquad\cdot\exp\left(-\frac{\|y - \sqrt{t}x_0\|_2^2}{2(1-t)}\right)\left\|\frac{y - \sqrt{t}x_0}{1-t}\right\|_2\mathrm{d}x_0\notag\\
&\lesssim \frac{2}{\sqrt{1-t}}\exp\left(\frac{(\|y\|_2 + \sqrt{t}R)^2}{2(1-t)}-cD^2+cd\right), 
\end{align*}
where $c$ is a universal constant.

By choosing 
$$
D = C \left(\frac{\|y\|_2 + \sqrt{t}R}{\sqrt{1-t}} + d\right)
$$
for some constant $C > 0$ large enough, we have 
$$
\big\|\nabla\log p_{X_{1-t}}(y)\big\|_2 \le \frac{2D}{\sqrt{1-t}}\lesssim \frac{\|y\|_2 + \sqrt{t}R}{1-t} + \frac{d}{\sqrt{1-t}}.
$$
Similarly, we could derive that
$$
\big\|\nabla\log p_{X_{1-t}\mymid c}(y\mymid c)\big\|_2 \le \frac{2D}{\sqrt{1-t}}\lesssim \frac{\|y\|_2 + \sqrt{t}R}{1-t} + \frac{d}{\sqrt{1-t}}.
$$

%\begin{align*}
%\frac{\partial^2}{\partial t^2}p_{c\mymid X_{1-t}}(c\mymid y)^{-1}
%\end{align*}
%\begin{align*}
%\frac{\partial^k}{\partial t^k}p_{c\mymid X_{1-t}}(c\mymid y)^{-1}
%\nabla^k p_{c\mymid X_{1-t}}(c\mymid y)^{-1}
%\nabla^k \frac{\partial}{\partial t}p_{c\mymid X_{1-t}}(c\mymid y)^{-1}
%\end{align*}

\subsection{Proof of Claim~\eqref{eq:diff}}
\label{sec:proof-diff}

We provide a detailed proof of Claim~\eqref{eq:diff} by analyzing the decomposition of the expectation.
We start by decomposing the expectation as follows:
\begin{align*}
&\quad \mathbb{E}\big[p_{c\mymid X_{1-t-\delta}}(c\mymid Y_{t+\delta})^{-1} - p_{c\mymid X_{1-t}}(c\mymid Y_t)^{-1} \mymid Y_{t} = y_{t}\big]\notag\\
&= \mathbb{E}\big[p_{c\mymid X_{1-t}}(c\mymid Y_{t+\delta})^{-1} \!-\! p_{c\mymid X_{1-t}}(c\mymid Y_t)^{-1} \mymid Y_{t} = y_{t}\big] \\
&\quad+ \mathbb{E}\big[p_{c\mymid X_{1-t-\delta}}(c\mymid Y_{t+\delta})^{-1} \!-\! p_{c\mymid X_{1-t}}(c\mymid Y_{t+\delta})^{-1} \mymid Y_{t} \!=\! y_{t}\big]
\end{align*}
In the following, we shall analyze these two terms separately.

\paragraph{Analysis of the first term.}
%\begin{align*}
%\mathbb{E}\big[p_{c\mymid X_{1-t}}(c\mymid Y_{t+\delta})^{-1} - p_{c\mymid X_{1-t}}(c\mymid Y_t)^{-1} \mymid Y_{t} = y_{t}\big]
%\end{align*}
Applying Ito's formula gives us
\begin{align}\label{eq:proof-diff-firstterm}
p_{c\mymid X_{1-t}}(c\mymid Y_{t+\delta})^{-1} - p_{c\mymid X_{1-t}}(c\mymid Y_{t})^{-1}
&= \int_t^{t+\delta} \bigg\{
\frac{1}{2s}\mathsf{Tr}\Big(\nabla^2 p_{c\mymid X_{1-t}}(c\mymid Y_{s})^{-1}\Big)\mathrm{d}s \notag \\
&\quad+ \nabla p_{c\mymid X_{1-t}}(c\mymid Y_{s})^{-1}
\cdot \Big(\Big(\frac{1}{2}Y_{s} + \nabla\log p_{X_{1-s}\mymid c}(Y_{s}\mymid c)\Big)\frac{\mathrm{d}s}{s} + \frac{1}{\sqrt{s}}\mathrm{d}B_{s}\Big) \bigg\}.
\end{align}

We further decompose the first term by using Ito's formula again as
\begin{align}
&\quad \mathsf{Tr}\Big(\nabla^2 p_{c\mymid X_{1-t}}(c\mymid Y_{s})^{-1}\Big)
- \mathsf{Tr}\Big(\nabla^2 p_{c\mymid X_{1-t}}(c\mymid Y_{t})^{-1}\Big)\notag\\
&= \int_t^{s} \bigg\{
\frac{1}{2r}\mathsf{Tr}\Big(\nabla^2 \mathsf{Tr}\Big(\nabla^2 p_{c\mymid X_{1-t}}(c\mymid Y_{r})^{-1}\Big)\Big)\mathrm{d}r \notag\\
&\qquad+ \nabla \mathsf{Tr}\Big(\nabla^2 p_{c\mymid X_{1-t}}(c\mymid Y_{r})^{-1}\Big)\cdot \Big(\Big(\frac{1}{2}Y_{r}  + \nabla\log p_{X_{1-r}\mymid c}(Y_{r}\mymid c)\Big)\frac{\mathrm{d}r}{r} + \frac{1}{\sqrt{r}}\mathrm{d}B_{r}\Big) \bigg\}.\label{eq:proof-temp-1}
\end{align}
According to bound \eqref{eq:bound-1}, we have
\begin{align}\label{eq:proof-temp-2}
\mathbb{E}\Big[\mathsf{Tr}\Big(\nabla^2 \mathsf{Tr}\Big(\nabla^2 p_{c\mymid X_{1-t}}(c\mymid Y_{r})^{-1}\Big)\Big) \mymid Y_{t} = y_{t}\Big]\le \mathbb{E}\Big[\exp(C_{r, 4, R} + C_{r, 4, R}\|Y_r\|_2^2) \mymid Y_{t} = y_{t}\Big]
< \infty
\end{align}
and
\begin{align}\label{eq:proof-temp-3}
\mathbb{E}\Big[\nabla \mathsf{Tr}\Big(\nabla^2 p_{c\mymid X_{1-t}}(c\mymid Y_{r})^{-1}\Big) 
\cdot \Big(\Big(\frac{1}{2}Y_{r} + \nabla\log p_{X_{1-r}\mymid c}(Y_{r}\mymid c)\Big) \mymid Y_{t} = y_{t}\Big]
< \infty.
\end{align}
Inserting \eqref{eq:proof-temp-2} and \eqref{eq:proof-temp-3} into \eqref{eq:proof-temp-1}, we have for $t\le s\le t+\delta$,
\begin{align}\label{eq:proof-temp-4}
&\quad\mathsf{Tr}\Big(\nabla^2 p_{c\mymid X_{1-t}}(c\mymid Y_{s})^{-1}\Big) = \mathsf{Tr}\Big(\nabla^2 p_{c\mymid X_{1-t}}(c\mymid Y_{t})^{-1}\Big) + O(\delta).
\end{align}
Similarly, we could get that for $t\le s\le t+\delta$,
\begin{align}\label{eq:proof-temp-5}
&\quad \mathbb{E}\Big[\nabla p_{c\mymid X_{1-t}}(c\mymid Y_{s})^{-1}
\cdot \Big(\Big(\frac{1}{2}Y_{s}+ \nabla\log p_{X_{1-s}\mymid c}(Y_{s}\mymid c)\Big) \mymid Y_{t} = y_{t}\Big] \notag\\
&= \nabla p_{c\mymid X_{1-t}}(c\mymid y_{t})^{-1}
\cdot \Big(\Big(\frac{1}{2}y_{t}+ \nabla\log p_{X_{1-t}\mymid c}(y_{t}\mymid c)\Big)
+ O(\delta).
\end{align}
Inserting \eqref{eq:proof-temp-4} and \eqref{eq:proof-temp-5} into \eqref{eq:proof-diff-firstterm}, we have
\begin{align*}
&\quad\frac{1}{\delta}\mathbb{E}\left[p_{c\mymid X_{1-t}}(c\mymid Y_{t+\delta})^{-1} - p_{c\mymid X_{1-t}}(c\mymid Y_{t})^{-1}|Y_t=y_t\right]\notag\\
%= \frac{1\delta}\int_t^{t+\delta} \bigg\{
%\frac{1}{2t}\mathsf{Tr}\Big(\nabla^2 p_{c\mymid X_{1-t}}(c\mymid Y_{t})^{-1}\Big)\mathrm{d}s \\
%&\qquad\qquad\qquad\qquad+ \nabla p_{c\mymid X_{1-t}}(c\mymid Y_{t})^{-1}
%\cdot \Big(\Big(\frac{1}{2}Y_{t} + \nabla\log p_{X_{1-t}\mymid c}(Y_{t}\mymid c)\Big)\frac{\mathrm{d}s}{s} + \frac{1}{\sqrt{s}}\mathrm{d}B_{s}\Big) + O(\delta) \mathrm{d} s\bigg\}\notag\\
&=\frac{1}{2t}\mathsf{Tr}\Big(\nabla^2 p_{c\mymid X_{1-t}}(c\mymid y_{t})^{-1}\Big) + \nabla p_{c\mymid X_{1-t}}(c\mymid y_{t})^{-1} \Big(\Big(\frac{1}{2}y_{t} + \nabla\log p_{X_{1-t}\mymid c}(y_{t}\mymid c)\Big)\frac{1}{t} + O(\delta).
\end{align*}

\paragraph{Analysis of the second term.}
The second term can be expressed as:
\begin{align*}
&\quad\mathbb{E}\big[p_{c\mymid X_{1-t-\delta}}(c\mymid Y_{t+\delta})^{-1} - p_{c\mymid X_{1-t}}(c\mymid Y_{t+\delta})^{-1} \mymid Y_{t} = y_{t}\big]\notag\\
&= \mathbb{E}\bigg[\int_t^{t+\delta} \frac{\partial}{\partial s}p_{c\mymid X_{1-s}}(c\mymid Y_{t+\delta})^{-1}\mathrm{d}s \mymid Y_{t} = y_{t}\bigg].
\end{align*}
Similar to the analysis of the first term, we notice that 
\begin{align*}
&\quad \frac{\partial}{\partial s}p_{c\mymid X_{1-s}}(c\mymid y)^{-1} - \frac{\partial}{\partial t}p_{c\mymid X_{1-t}}(c\mymid y)^{-1}= \int_t^{s} \frac{\partial^2}{\partial r^2}p_{c\mymid X_{1-r}}(c\mymid y)^{-1}\mathrm{d}r,
\end{align*}
and according to \eqref{eq:bound-2},
\begin{align*}
\mathbb{E}\bigg[\frac{\partial^2}{\partial r^2}p_{c\mymid X_{1-r}}(c\mymid Y_{t+\delta})^{-1} \mymid Y_{t} = y_{t}\bigg]
< \infty.
\end{align*}
Thus we have
\begin{align*}
% &\mathbb{E}_{y \sim Y_{t+\delta}}\bigg[\frac{\partial}{\partial t}p_{c\mymid X_{1-t}}(c\mymid y)^{-1} \mymid Y_{t} = y_{t}\bigg]\notag\\
% &\quad - \mathbb{E}_{y \sim Y_{t}}\bigg[\frac{\partial}{\partial t}p_{c\mymid X_{1-t}}(c\mymid y)^{-1} \mymid Y_{t} = y_{t}\bigg] \\
% &= \mathbb{E}_{y \sim Y_{t}}\bigg[\int_{t}^{t+\delta}\bigg\{\nabla\frac{\partial}{\partial t}p_{c\mymid X_{1-t}}(c\mymid Y_{s})^{-1} \notag\\
% &\quad \cdot \Big(\frac{1}{2}Y_{s} + \nabla\log p_{X_{1-s}\mymid c}(Y_{s}\mymid c)\Big)\frac{\mathrm{d}s}{s} \\
% &\qquad\qquad+ \frac{1}{2s}\mathsf{Tr}\Big(\nabla^2 \frac{\partial}{\partial t}p_{c\mymid X_{1-t}}(c\mymid Y_{s})^{-1}\Big)\mathrm{d}s\bigg\}
% \mymid Y_{t} = y_{t}\bigg] \\
% &= O(\delta).\nonumber\\
&\quad\frac{1}{\delta}\mathbb{E}\big[p_{c\mymid X_{1-t-\delta}}(c\mymid Y_{t+\delta})^{-1} - p_{c\mymid X_{1-t}}(c\mymid Y_{t+\delta})^{-1} \mymid Y_{t} = y_{t}\big]= \frac{\partial}{\partial t}p_{c\mymid X_{1-t}}(c\mymid y_t)^{-1} + O(\delta).
\end{align*}

%\paragraph{Part III.}
%\begin{align*}
%\mathrm{d}p_{c\mymid X_{1-t}}(c\mymid Y_t)^{-1}
%&= \frac{\partial}{\partial t}p_{c\mymid X_{1-t}}(c\mymid y)^{-1}\mymid_{y = Y_t} + \frac{1}{2t}\mathsf{Tr}\Big(\nabla^2 p_{c\mymid X_{1-t}}(c\mymid Y_{t})^{-1}\Big)\mathrm{d}t \\
%&\qquad\qquad\qquad+ \nabla p_{c\mymid X_{1-t}}(c\mymid Y_{t})^{-1}
%\cdot \Big(\Big(\frac{1}{2}Y_{t} + \nabla\log p_{X_{1-t}\mymid c}(Y_{t}\mymid c)\Big)\frac{\mathrm{d}t}{s} + \frac{1}{\sqrt{t}}\mathrm{d}B_{t}\Big)
%\end{align*}
Combining the above two relations, we could get our desired result.

% \begin{figure}
% \centering
% \includegraphics[scale=0.8]{../fig-exp.pdf}
% \caption{left: Ratio of samples with improved classifier probabilities for different guidance scales $w$; right: Expectation of $-p_{c\mymid X_0}(1\mymid Y_0^w)^{-1}$ for varying $w$.}
% \label{fig:exp}
% \end{figure}
%Plot: (1) ratio of samples with increased $p_{c\mymid X_0}(1\mymid x)$ for different $w$;
%(2) expectation of $-p_{c\mymid X_0}(1\mymid x)^{-1}$ for different $w$.

\section{Discussion}
\label{sec:discussion}

In this paper, we present a theoretical analysis of the impact of guidance in diffusion models under general data distributions.
Specifically, we demonstrate that guidance in the continuous-time process enhances sampling by increasing the average quality of generated samples, as measured by classifier probabilities. 
This result provide a theoretical foundation for the empirical success of guidance methods.
Interestingly, our results show that guidance improves the average reciprocal of classifier probabilities rather than improving every sample individually, implying that some samples may degrade in quality. This observation motivates future work on adaptive guidance strategies for more uniform performance. 
%Specifically, we demonstrate that guidance in the continuous-time process can enhance the sampling process by generating more high-quality samples --- those associated with higher classifier probabilities --- in the average sense.
%Notice that the guidance improves the averaged reciprocal of the classifier probability rather than the classifier probability of each individual sample.
%Our result also suggests that while guidance improves overall sample quality, it may lead to a decline in quality for a small subset of samples.
%Additionally, we prove that the practical discrete-time process converges to the above analyzed continuous-time process, as the number of iterations goes to infinity.
%These results provide a theoretical foundation for the empirical success of guidance methods, and encourages the development of adaptive guidance methods that achieve more uniform performance gains.
In the future, we are interested in extending these results to the concept of Inception Score (IS), demonstrating similar findings when the weights used in IS are applied.

% In this paper, the convergence analysis in Theorem \ref{thm:convergence} is included primarily for completeness.
% The dependencies on $d$, $L$ and $\varepsilon$ may not be optimal, and the smoothness condition might not be necessary. Future research could focus on establishing tighter bounds or analyzing under more general bounds, to broaden the applicability and improve the convergence rate.
%In addition, we are interested in extending these results to the concept of Inception Score (IS), demonstrating similar findings when the weights used in IS are applied.

\appendix

\section{Discretization and robustness analysis}

Consider that practical algorithms operate in discrete time and are subject to score estimation errors, we provide a supplementary analysis of the discretization error and estimation error for completeness.
Specifically, we aim to show the discrete-time process in \eqref{eq:guidance} closely approximates the continuous-time process in \eqref{eq:cont-guidance}, thereby validating the observation from Theorem~\ref{thm:main} in practical settings.
Since our primary focus is on the efficiency of diffusion guidance rather than establishing a convergence theory, the bounds and conditions derived here may not be tight.

In the following, we shall use $Y_t^{w, \mathsf{cont}}$ to denote the continuous process of~\eqref{eq:cont-guidance} in order to distinguish with~\eqref{eq:guidance},
and let
\begin{align}
\overline{\alpha}_n := \prod_{k = 1}^n \alpha_k,
\quad\text{with }\alpha_k := 1-\beta_k
\end{align}
satisfying
\begin{subequations}\label{eq:learning-rate}
\begin{align}
% {\alpha}_{1} &= 1-\frac{1}{N^{c_0}},\\
% {\alpha}_{n} &= 1-\frac{c_1\log N}{N}\min\left\{\beta_1\left(1+\frac{c_1\log N}{N}\right)^n,1\right\},
\overline{\alpha}_{N} &= \frac{1}{N^{c_0}},\\
\overline{\alpha}_{n-1} &= \overline{\alpha}_{n} + \frac{c_1\overline{\alpha}_{n}(1-\overline{\alpha}_{n})\log N}{N},
%X_{1-\overline{\alpha}_n} \overset{d}{=}\sqrt{\overline{\alpha}_n} X_0 + \sqrt{1-\overline{\alpha}_n} Z \overset{d}{=} X_n,
\end{align}
\end{subequations}
where $c_0$ and $c_1$ are constants.

Before presenting the analysis result, we make the following assumptions.
The first assumption states that faithful estimates of the score functions $s_{n}^{\star}(\cdot)$ and $s_{n}^{\star}(\cdot|c)$ are available for all intermediate steps $n$, as follows:
\begin{assumption}\label{ass:estimation}
We assume access to estimates $s_n(Y_{\overline{\alpha}_n}^{w, \mathsf{cont}})$ and $s_n(Y_{\overline{\alpha}_n}^{w, \mathsf{cont}}\mymid c)$ for each $s^{\star}_n(Y_{\overline{\alpha}_n}^{w, \mathsf{cont}})$ and $s^{\star}_n(Y_{\overline{\alpha}_n}^{w, \mathsf{cont}}\mymid c)$ with the averaged $\ell_2$ score estimation error as
\begin{subequations}
\begin{align}
&\frac{1}{N}\sum_{n = 1}^N\mathbb{E} \Big[\big\|s_{n}(Y_{\overline{\alpha}_n}^{w, \mathsf{cont}}\mymid c)- \nabla\log p_{X_{1-\overline{\alpha}_n}\mymid c}(Y_{\overline{\alpha}_n}^{w, \mathsf{cont}}\mymid c)\big\|_2^2\Big] \le \varepsilon_{\mathsf{score}}^2; \\
&\frac{1}{N}\sum_{n = 1}^N\mathbb{E} \Big[\big\|s_{n}(Y_{\overline{\alpha}_n}^{w, \mathsf{cont}})- \nabla\log p_{X_{1-\overline{\alpha}_n}}(Y_{\overline{\alpha}_n}^{w, \mathsf{cont}})\big\|_2^2\Big] \le \varepsilon_{\mathsf{score}}^2.
\end{align}
\end{subequations}
\end{assumption}

We further assume that the sample $Y_t^{w, \mathsf{cont}}$, the score function $\nabla\log p_{X_{1-t}}(Y_t^{w, \mathsf{cont}})$, and the conditional score function $\nabla\log p_{X_{1-t}\mymid c}(Y_t^{w, \mathsf{cont}}\mymid c)$ have bounded second-order moment, which is stated in the following lemma.
\begin{assumption} \label{ass:bound}
There exists some quantity $R$, such that the sum of the second-order moment of the following three random vectors are bounded by $R^2$, that is,
\begin{align}
&\mathbb{E} \Big[\big\|Y_t^{w, \mathsf{cont}}\big\|_2^2 + \big\|\nabla\log p_{X_{1-t}}(Y_t^{w, \mathsf{cont}})\big\|_2^2+ \big\|\nabla\log p_{X_{1-t}\mymid c}(Y_t^{w, \mathsf{cont}}\mymid c)\big\|_2^2\Big] \le R^2.
\end{align}
\end{assumption}

In addition, we consider the case with smooth score functions in this paper, which is stated below.
\begin{assumption} \label{ass:Lip}
Assume that $\nabla\log p_{X_{t}}(x)$ are Lipschitz for all $0<t<1$ such that
\begin{align}
\big\|\nabla\log p_{X_{t}}(x_1)\! - \!\nabla\log p_{X_{t}}(x_2)\big\|_2 \le L\|x_1 \!-\! x_2\|_2.
\end{align}
\end{assumption}

With the above assumptions, We could establish that the discrete-time process converges to the continuous-time process measured by the KL divergence. The proof is postponed to Section \ref{appendix:proof-thm2}.

\begin{theorem} \label{thm:convergence}
Suppose that Assumptions \ref{ass:estimation}, \ref{ass:bound}, and \ref{ass:Lip} hold true.  
Then the sampling process \eqref{eq:guidance} with the learning rate schedule \eqref{eq:learning-rate} satisfies
\begin{align}%Y_{1-\overline{\alpha}_N}^{w, \mathsf{cont}}
&\mathsf{KL}(Y_{\overline{\alpha}_1}^{w, \mathsf{cont}}, Y_1^w) 
\le C\Big(\frac{(1+w^2)L^2d\log^3 N}{N} + \frac{(1+w^4)L^2R^2\log^4 N}{N^2} + (1+w^2)\varepsilon_{\mathsf{score}}^2\log N\Big)
\end{align}
for some constant $C > 0$ large enough,
where $Y_{\overline{\alpha}_1}^{w, \mathsf{cont}}$ and $Y_1^w$ are defined in \eqref{eq:cont-guidance} and \eqref{eq:guidance}, respectively.
\end{theorem}

This theorem proves that, after a sufficiently large number of iterations $N$, the sample distribution of the discrete-time process $Y_n^w$ converges to that of the continuous-time process $Y_{\overline{\alpha}_1}^{w, \mathsf{cont}}$.
The latter corresponds to data contaminated by noise with variance $1-\overline{\alpha}_1$.
According to Theorem \ref{thm:convergence},
the sampling process \eqref{eq:guidance} with the learning rate schedule \eqref{eq:learning-rate} satisfies
\begin{align*}
\mathbb{E}[p(c | Y_1^w)^{-1}] \le \mathbb{E}[p(c | Y_{\overline{\alpha}_1}^{w, \mathsf{cont}})^{-1}] + \mathbb{E}[(p(c | Y_{1}^{w})^{-1}-1)\mathds{1}(p(c | Y_{1}^{w})^{-1} > \tau)],
\end{align*}
where $\tau$ is defined as the largest value satisfying
\begin{align*} 
\mathsf{TV}(Y_{\overline{\alpha}_1}^{w, \mathsf{cont}}, Y_1^w) \le \mathbb{P}(p(c | Y_{1}^{w})^{-1} > \tau).
\end{align*}
This further implies the following relative influence from discretization, the ratio between the improvements of $Y_1^w$ and $Y_{\overline{\alpha}_1}^{w, \mathsf{cont}}$ over $X_{\overline{\alpha}_1} = Y_{\overline{\alpha}_1}^{0, \mathsf{cont}}$, obeys
\begin{align}\label{eq:relative-error} 
\frac{\mathbb{E}[p(c | Y_{\overline{\alpha}_1}^{0, \mathsf{cont}})^{-1}] - \mathbb{E}[p(c | Y_1^w)^{-1}]}{\mathbb{E}[p(c | Y_{\overline{\alpha}_1}^{0, \mathsf{cont}})^{-1}] - \mathbb{E}[p(c | Y_{\overline{\alpha}_1}^{w, \mathsf{cont}})^{-1}]}
\ge 1 - \frac{\mathbb{E}[(p(c | Y_{1}^{w})^{-1}-1)\mathds{1}(p(c | Y_1^w)^{-1} > \tau)]}{\mathbb{E}[p(c | Y_{\overline{\alpha}_1}^{0, \mathsf{cont}})^{-1}] - \mathbb{E}[p(c | Y_{\overline{\alpha}_1}^{w, \mathsf{cont}})^{-1}]}.
\end{align}

\subsection{Numerical validation}

For different values of $\mathsf{TV}(Y_{\overline{\alpha}_1}^{w, \mathsf{cont}}, Y_1^w)$, we empirically validate the aforementioned result on the ImageNet dataset. 
Specifically, we generate $2\times 10^4$ samples $Y_1^w$ under various guidance level $w$ and their counterparts $Y_0^w$ without guidance by using a pre-trained diffusion model \citep{diffusioncode}, and evaluate the classifier probability $p(c|Y_1^w)$ and $p(c|Y_1^0)$ by using the Inception v3 classifier \citep{szegedy2016rethinking}.
Finally, we evaluate the relative error in \eqref{eq:relative-error}.
Here we use $\mathbb{E}[p(c | Y_{1}^{0})^{-1}]-\mathbb{E}[p(c | Y_{1}^{w})^{-1}]$ as an estimate of $\mathbb{E}[p(c | Y_{\overline{\alpha}_1}^{0, \mathsf{cont}})^{-1}] - \mathbb{E}[p(c | Y_{\overline{\alpha}_1}^{w, \mathsf{cont}})^{-1}]$, and calculate the ratio of empirical average
$$
\frac{\mathbb{E}[(p(c | Y_{1}^{w})^{-1}-1)\mathds{1}(p(c | Y_1^w)^{-1} > \tau)]}{\mathbb{E}[p(c | Y_{1}^{0})^{-1}]-\mathbb{E}[p(c | Y_{1}^{w})^{-1}]}.
$$

The results are presented in the following table for various values of the TV distance and $w$, which indicate that the relative error remains small, particularly for practical choices of $w \ge 1$.

\begin{table}[htbp]
\centering
\caption{Empirical values of $\frac{\mathbb{E}[(p(c | Y_{1}^{w})^{-1}-1)\mathds{1}(p(c | Y_1^w)^{-1} > \tau)]}{\mathbb{E}[p(c | Y_{1}^{0})^{-1}]-\mathbb{E}[p(c | Y_{1}^{w})^{-1}]}$ under different $w$ and TV.}
\label{tab:tv-guidance}
\begin{tabular}{c|cccccccc}
\toprule
\textsf{TV} & $w=0.2$ & $0.4$ & $0.6$ & $0.8$ & $1$ & $2$ & $3$ & $4$ \\
\midrule
$0.30$ & 0.447 & 0.196 & 0.115 & 0.085 & 0.029 & 0.006 & 0.006 & 0.002 \\
$0.10$ & 0.440 & 0.194 & 0.114 & 0.085 & 0.029 & 0.006 & 0.005 & 0.002 \\
\bottomrule
\end{tabular}
\end{table}

% \begin{align*}
% \begin{array}{c| c c c c c c c c c}
% \hline
% \hline
%  \mathsf{TV}& w=0.2 & 0.4 & 0.6 & 0.8 & 1 & 2 & 3 & 4 \\
% \hline
% 0.30 & 0.447 & 0.196 & 0.115 & 0.085 & 0.029 & 0.006 & 0.006 & 0.002  \\
% 0.10 & 0.440 & 0.194 & 0.114 & 0.085 & 0.029 & 0.006 & 0.005 & 0.002 \\
% \hline
% \end{array}
% \end{align*}
%By combining Theorems \ref{thm:main} and \ref{thm:convergence}, we demonstrate that the observation from Theorem~\ref{thm:main} also holds in practical discrete-time settings. 
%This ensures the validity of the theoretical observations in real-world implementations.

\subsection{Proof of Theorem~\ref{thm:convergence}}
\label{appendix:proof-thm2}

Here, we provide a brief sketch for this result.
With similar analysis as~\citet[Section 5]{chen2022sampling},
\begin{align}
&\quad \mathsf{KL}(Y_{\overline{\alpha}_1}^{w, \mathsf{cont}}, Y_1^w) \notag\\
&\le \sum_{n=2}^N \mathbb{E} \int_{\overline{\alpha}_n}^{\overline{\alpha}_{n-1}} 
\big\|(1+w)[s_{n}(Y_{\overline{\alpha}_n}^{w, \mathsf{cont}}\mymid c) - \nabla\log p_{X_{1-t}\mymid c}(Y_t^{w, \mathsf{cont}}\mymid c)] \notag\\
&\qquad- w[s_{n}(Y_{\overline{\alpha}_n}^{w, \mathsf{cont}}) - \nabla\log p_{X_{1-t}}(Y_t^{w, \mathsf{cont}})]\big\|_2^2\frac{\mathrm{d}t}{t}+ \mathsf{KL}(Y_{\overline{\alpha}_N}^{w, \mathsf{cont}}, Y_N^w).
\label{eq:decomposition}
\end{align}
% \begin{align}
% \mathsf{KL}(Y_{\overline{\alpha}_1}^{w, \mathsf{cont}}, Y_0^w) 
% &\le \sum_{n=2}^N \mathbb{E} \int_{\overline{\alpha}_n}^{\overline{\alpha}_{n-1}} 
% \big\|(1+w)[s_{X_{1-\overline{\alpha}_n}\mymid c}(Y_{\overline{\alpha}_n}^{w, \mathsf{cont}}\mymid c) - \nabla\log p_{X_{1-t}\mymid c}(Y_t^{w, \mathsf{cont}}\mymid c)] \notag\\
% &\qquad- w[s_{X_{1-\overline{\alpha}_n}}(Y_{\overline{\alpha}_n}^{w, \mathsf{cont}}) - \nabla\log p_{X_{1-t}}(Y_t^{w, \mathsf{cont}})]\big\|_2^2\frac{\mathrm{d}t}{t}
% + \mathsf{KL}(Y_{\overline{\alpha}_N}^{w, \mathsf{cont}}, Y_N^w).
% \label{eq:decomposition}
% \end{align}
Then it can be shown that
\begin{align*}
&\quad\mathbb{E} \int_{\overline{\alpha}_n}^{\overline{\alpha}_{n-1}} 
\big\|s_{n}^{\star}(Y_{\overline{\alpha}_n}^{w, \mathsf{cont}}) - \nabla\log p_{X_{1-t}}(Y_t^{w, \mathsf{cont}})\big\|_2^2\frac{\mathrm{d}t}{t}\notag\\
&\le L^2\mathbb{E} \int_{\overline{\alpha}_n}^{\overline{\alpha}_{n-1}} 
\big\|Y_{\overline{\alpha}_n}^{w, \mathsf{cont}} - Y_t^{w, \mathsf{cont}}\big\|_2^2\frac{\mathrm{d}t}{t} \\
&\le L^2\mathbb{E} \int_{\overline{\alpha}_n}^{\overline{\alpha}_{n-1}} 
\bigg\|\int_{\overline{\alpha}_n}^{t}\bigg\{
\Big(\frac{Y_{\tau}^{w, \mathsf{cont}}}{2} + (1+w)\nabla\log p_{X_{1-{\tau}}\mymid c}(Y_{\tau}^{w, \mathsf{cont}}\mymid c) - w\nabla\log p_{X_{1-{\tau}}}(Y_{\tau}^{w, \mathsf{cont}})\Big)\frac{\mathrm{d}{\tau}}{{\tau}} + \frac{\mathrm{d}B_{\tau}}{\sqrt{{\tau}}}\bigg\}\bigg\|_2^2\frac{\mathrm{d}t}{t} \\
&\lesssim L^2((1+w)^2R^2(1-\alpha_n) + d)(1-\alpha_n)^2.
\end{align*}
Inserting the above relation, Assumption~\ref{ass:estimation}, and Assumption~\ref{ass:bound} into~\eqref{eq:decomposition} leads to our desired result.

\section{Basis calculations of GMM}
\label{appendix:GMM}

Consider a GMM defined as:
\begin{align}\label{eq:GMM}
X_0 \sim \sum_{k = 1}^K \pi_k\mathcal{N}(\mu_k, 1),
\end{align}
where $\pi_k$ is the mixing coefficient of the 
$k$-th component, and $\mu_k$ is its mean.
By Lemma \ref{lem:cont}, we have
\begin{align*}
X_{1-\overline{\alpha}_n} &\sim \sum_{k = 1}^K \pi_k\mathcal{N}\big(\sqrt{\overline{\alpha}_n}\mu_k, 1\big)\notag\\
p_{X_{1-\overline{\alpha}_n}}(x) &= \sum_{k = 1}^K \pi_k(2\pi)^{-1/2}\exp\bigg(-\frac{(x - \sqrt{\overline{\alpha}_n}\mu_k)^2}{2}\bigg).
\end{align*}
The gradient of the log-density $\log p_{X{1-\overline{\alpha}_n}}(x)$ can be computed as:
\begin{align}\label{eq:GMM-nabla-px}
\nabla \log p_{X_{1-\overline{\alpha}_n}}(x) 
&= \frac{\nabla p_{X_{1-\overline{\alpha}_n}}(x)}{p_{X_{1-\overline{\alpha}_n}}(x)} = -\sum_{k = 1}^K \pi_k^n \big(x - \sqrt{\overline{\alpha}_n}\mu_k\big) = -x+\sqrt{\overline{\alpha}_n}\sum_{k = 1}^K \pi_k^n \mu_k,
\end{align}
where
\begin{align*}
\pi_k^n = \frac{\pi_k\exp\big(-\frac{(x - \sqrt{\overline{\alpha}_n}\mu_k)^2}{2}\big)}{\sum_{i = 1}^K \pi_i\exp\big(-\frac{(x - \sqrt{\overline{\alpha}_n}\mu_i)^2}{2}\big)}.
\end{align*}
Using this setup for specific cases $(K=2,3)$ leads to 
\begin{align}
&\quad \nabla \log p_{X_{1-\overline{\alpha}_n}\mymid c}(x\mymid 1) = -x + \frac{\sqrt{\overline{\alpha}_n}(1 - \exp(-2\sqrt{\overline{\alpha}_n}x))}{1 + \exp(-2\sqrt{\overline{\alpha}_n}x)}; \label{eq:score-GMM-1}\\
&\quad \nabla \log p_{X_{1-\overline{\alpha}_n}}(x)= -x  + \frac{\sqrt{\overline{\alpha}_n}(1 - \exp(-2\sqrt{\overline{\alpha}_n}x))}{1 + \exp(-2\sqrt{\overline{\alpha}_n}x) + 2\exp\big(\frac{\overline{\alpha}_n}{2} - \sqrt{\overline{\alpha}_n}x\big)}.\label{eq:score-GMM-2}
\end{align}
Additionally, the classifier probability $p_{c\mymid X_{1-\overline{\alpha}_n}}(1\mymid x)$ is given by
\begin{align}\label{eq:llh-GMM}
p_{c\mymid X_{1-\overline{\alpha}_n}}(1\mymid x) 
&= \frac{p_{X_{1-\overline{\alpha}_n}\mymid c}(x\mymid c)p(c)}{p_{X_{1-\overline{\alpha}_n}}(x)}= \frac{1 + \exp(-2\sqrt{\overline{\alpha}_n}x)}{1 + \exp(-2\sqrt{\overline{\alpha}_n}x) + 2\exp\big(\frac{\overline{\alpha}_n}{2} - \sqrt{\overline{\alpha}_n}x\big)}.
\end{align}

\bibliographystyle{apalike}
\bibliography{ICML/refs}

\end{document}